\documentclass{article}  

\usepackage[authoryear]{natbib}
\setcitestyle{authoryear}
\usepackage{amsmath} 
\usepackage{amsthm} 
\usepackage{amssymb}  

\usepackage{pgfplots}
\usepackage{abbrevs}

\usepackage{multirow}

\usetikzlibrary{arrows}
\usetikzlibrary{positioning}
\usepackage{color}
\definecolor{envNodeColor}{rgb}{0.6,0.6,0.9}
\definecolor{envArrowColor}{rgb}{0.2,0.2,0.8}
\definecolor{intNodeColor}{rgb}{0.5,0.8,0.5}
\definecolor{intArrowColor}{rgb}{0.1,0.8,0.1}
\definecolor{boolNodeColor}{rgb}{0.8,0.5,0.5}
\definecolor{boolArrowColor}{rgb}{0.8,0.2,0.2}

\newcommand{\seq}{\bar}

\newcommand{\hist}[1]{o_{\le #1}a_{<#1}}
\newcommand{\althist}[1]{o'_{\le #1}a'_{<#1}}
\newcommand{\dohist}[1]{o_{\le #1} \mid do(a_{<#1})}
\newcommand{\epair}{\left(\EE^1[U^1;\pi],\EE^2[U^2;\pi]\right)}
\newcommand{\bool}{B}

\newcommand{\fig}[1]{Figure~\ref{fig:#1}}
\newcommand{\eqn}[1]{Equation~\ref{eqn:#1}}
\newcommand{\lem}[1]{Lemma~\ref{lem:#1}}
\newcommand{\thm}[1]{Theorem~\ref{thm:#1}}
\newcommand{\defn}[1]{Definition~\ref{defn:#1}}
\newcommand{\prop}[1]{Proposition~\ref{prop:#1}}
\newcommand{\cor}[1]{Corollary~\ref{cor:#1}}
\newcommand{\sect}[1]{Section~\ref{sec:#1}}

\DeclareMathOperator*{\argmax}{argmax}
\DeclareMathOperator{\image}{image}

\newtheorem{theorem}{Theorem}
\newtheorem{proposition}[theorem]{Proposition}
\newtheorem{lemma}[theorem]{Lemma}
\newtheorem{corollary}[theorem]{Corollary}

\newtheorem{definition}[theorem]{Definition}

\title{\LARGE \bf
Toward negotiable reinforcement learning: shifting priorities in Pareto optimal sequential decision-making
}

\author{Andrew Critch
\thanks{Machine Intelligence Research Institute}
\thanks{UC Berkeley, Center for Human Compatible AI}%
}

\begin{document}

\maketitle
\thispagestyle{empty}
\pagestyle{empty}

\begin{abstract}
Existing multi-objective reinforcement learning (MORL) algorithms do not account for objectives that arise from players with differing beliefs.  Concretely, consider two players with different beliefs and utility functions who may cooperate to build a machine that takes actions on their behalf.  A representation is needed for how much the machine's policy will prioritize each player's interests over time.  Assuming the players have reached common knowledge of their situation, this paper derives a recursion that any Pareto optimal policy must satisfy.  Two qualitative observations can be made from the recursion: the machine must (1) use each player's own beliefs in evaluating how well an action will serve that player's utility function, and (2) shift the relative priority it assigns to each player's expected utilities over time, by a factor proportional to how well that player's beliefs predict the machine's inputs.  Observation (2) represents a substantial divergence from na\"{i}ve linear utility aggregation (as in Harsanyi's utilitarian theorem, and existing MORL algorithms), which is shown here to be inadequate for Pareto optimal sequential decision-making on behalf of players with different beliefs.
\end{abstract}


\section{Introduction}

It has been argued that the first AI systems with generally super-human cognitive abilities will play a pivotal decision-making role in directing the future of civilization \citep{bostrom2014superintelligence}.  If that is the case, an important question will arise: \emph{Whose values will the first super-human AI systems serve?}  Since safety is a crucial consideration in developing such systems, assuming the institutions building them come to understand the risks and the time investments needed to address them \citep{baum2016promotion}, they will have a large incentive to cooperate in their design rather than racing under time-pressure to build competing systems \citep{armstrong2016racing}.

Therefore, consider two nations---allies or adversaries---who must decide whether to cooperate in the deployment of an extremely powerful AI system.  Implicitly or explicitly, the resulting system would have to strike compromises when conflicts arise between the wishes of those nations.  How can they specify the degree to which that system would be governed by the distinctly held principles of each nation?  More mundanely, suppose a couple purchases a domestic robot.  How should the robot strike compromises when conflicts arise between the commands of its owners?

It is already an interesting and difficult problem to robustly align an AI system's values with those of a single \emph{single} human (or a group of humans in close agreement).  Inverse reinforcement learning (IRL) \citep{russell1998learning} \citep{ng2000algorithms} \citep{abbeel2004apprenticeship} and cooperative inverse reinforcement learning (CIRL) \citep{hadfield2016cooperative} represent successively realistic early approaches to this problem.  But supposing some adequate solution eventually exists for aligning the values of a machine intelligence with a single human decision-making unit, how should the values of a system serving \emph{multiple} decision-makers be ``aligned''?

One might hope to specify some extremely compelling ethical principle that everyone would immediately accept.  Realistically, however, disagreements will always exist.  Consider the general case of two parties---perhaps states, companies, or individuals---who might cooperatively build or purchase an AI system to serve them both.\footnote{The results of this paper all generalize directly from $2$ to $n$ players, but for concreteness of exposition, the two-player case is prioritized.}   If the parties cannot reach sufficient agreement as to what policy the AI should follow, cooperation may be less attractive than obtaining separate AI systems, one for each party.  At the individual level, non-cooperation could mean domestic disputes between domestic robots.  At the state level, it could mean an arms race between nations competing under time pressure to develop ever more powerful militarized AI systems, affording each nation less time to ensure the safety and validity of their respective systems.

Unless the prospect of cooperative AI ownership is made sufficiently attractive to the separate parties, the question of whose values the cooperatively owned system ``ought'' to serve is moot:  the parties will fall back on non-cooperative strategies---perhaps obtaining separate machines that will compete with each other---and the jointly owned system will not exist in the first place.   In addition, if the process of bargaining over the policy of a cooperatively owned system is difficult or complicated, the players are more likely to end negotiations and default to non-cooperative strategies. 

Conversely, if bargaining is made easier, players are more likely to reach cooperation.  The purpose of this paper is to begin formalizing the problem of negotiating over the policy of a machine intelligence, and to exhibit some early findings as to the nature of \emph{Pareto optimal policies}---policies which cannot be improved for one player without sacrifice by another---with the eventual aim of making cooperative outcomes easier to formulate, more attractive, and more likely to obtain.

\paragraph{Outline.} The paper is organized as follows.  Section 2 briefly outlines some standard choices of notation.  Section 3 formalizes the problem of obtaining a Pareto optimal policy for two distinct parties with common knowledge of distinct priors, derives a recursion that any such policy must follow, and contrasts that recursion with a more na\"{i}ve ``just add up a linear combination of the utility functions" approach.  The recursion implies two main qualitative insights about how a Pareto optimal policy, pursuant to a common-knowledge difference in opinion, must behave over time: (1) such a policy must (explicitly or implicitly) use each player's own beliefs in evaluating how well an action will serve that player's utility function, and (2) it must shift the relative priority it places on each player's expected utilities over time, by a factor proportional to how well that player's beliefs predict the machine's inputs.  Section 4 provides some further interpretation of these implications, in terms of bet-settling and moral realism.  Section 5 outlines subsequent work expected to be useful for enabling cooperative AI deployment in the future.  Finally, Section 6 provides concluding remarks targeted at readers who have finished the full paper.

\subsection{Related work}

\paragraph{Social choice theory.} The whole of social choice theory and voting theory may be viewed as an attempt to specify an agreeable formal policy to enact on behalf of a group.  Harsanyi's utility aggregation theorem \citep{harsanyi1980cardinal} suggests one form of solution: maximizing a linear combination of group members' utility functions.  The present work shows that this solution is inappropriate when players have different beliefs, and \thm{main} may be viewed as an extension of Harsanyi's form that accounts simultaneously for differing priors and the prospect of future observations.  Indeed, Harsanyi's form follows as a direct corollary of \thm{main} when players do share the same beliefs (\cor{harsanyi}).

\paragraph{Bargaining theory.} The formal theory of bargaining, as pioneered by Nash \citep{nash1950bargaining} and carried on by authors such as Myerson \citep{myerson1979incentive} \citep{myerson2013game} and Satterthwaite \citep{myerson1983efficient}, is also extremely topical.  Future investigation in this area might be aimed at generalizing their work to sequential decision-making settings, and this author recommends a focus on research specifically targeted at resolving conflicts.

\paragraph{Multi-agent systems.} There is ample literature examining multi-agent systems using sequential decision-making models.  \citet{shoham2008multiagent} survey various models of multiplayer games using an MDP to model each agent's objectives.  Chapter 9 of the same text surveys social choice theory, but does not account for sequential decision-making.  

\citet{zhang2014fairness} may be considered a sequential decision-making approach to social choice: they use MDPs to represent the decisions of players in a competitive game, and exhibit an algorithm for the players that, if followed, arrives at a Pareto optimal Nash equilibrium satisfying a certain fairness criterion.  Among the literature surveyed here, that paper is the closest to the present work in terms of its intended application: roughly speaking, achieving mutually desirable outcomes via sequential decision-making.  However, that work is concerned with an ongoing interaction between the players, rather than selecting a policy for a single agent to follow as in this paper.  


\paragraph{Multi-objective sequential decision-making.} There is also a good deal of work on Multi-Objective Optimization (MOO)  \citep{tzeng2011multiple}, including for sequential decision-making, where solution methods have been called Multi-Objective Reinforcement Learning (MORL).  For instance, \citet{gabor1998multi} introduce a MORL method called Pareto Q-learning for learning a set of a Pareto optimal polices for a Multi-Objective MDP (MOMDP).  \citet{soh2011evolving} define Multi-Reward Partially Observable Markov Decision Processes (MR-POMDPs), and use use genetic algorithms to produce non-dominated sets of policies for them.  \citet{roijers2015point} refer to the same problems as Multi-objective POMDPS (MOPOMDPs), and provide a bounded approximation method for the optimal solution set for all possible weightings of the objectives.  \citet{wang2014multi} surveys MORL methods, and contributes Multi-Objective Monte-Carlo Tree Search (MOMCTS) for discovering multiple Pareto optimal solutions to a multi-objective optimization problem.   \citet{wray2015multi} introduce Lexicographic Partially Observable Markov Decision Process (LPOMDPs), along with two accompanying solution methods.

However, none of these or related works address scenarios where the objectives are derived from players with differing beliefs, from which the priority-shifting phenomenon of \thm{main} arises.  Differing beliefs are likely to play a key role in negotiations, so for that purpose, the formulation of multi-objective decision-making adopted here is preferable.

\section{Notation}
The reader is invited to skip this section and refer back as needed; an effort has been made to use notation that is intuitive and fairly standard, following \citet{pearl2009causality}, \citet{hutter2003gentle}, and \citet{orseau2012space}.

Random variables are denoted by uppercase letters, e.g., $S_1$, and lowercase letters, e.g., $s_1$, are used as indices ranging over the values of a variable, as in the equation
\[
\EE[S_1] = \sum_{s_1} \PP(s_1)\cdot s_1.
\]

Sequences are denoted by overbars, e.g., given a sequence $(s_1,\ldots,s_n)$, $\seq s$ stands for the whole sequence.   Subsequences are denoted
by subscripted inequalities, so e.g., $s_{<4}$ stands for $(s_1,s_2,s_3)$,
and $s_{\le 4}$ stands for $(s_1,s_2,s_3,s_4)$.

\section{Two agents building a third}

Consider, informally, a scenario wherein two players --- perhaps individuals, companies, or states --- are considering cooperating to build or otherwise obtain a machine that will then interact with an environment on their behalf.\footnote{The results here all generalize from two players to $n$ players being combined successively in any order, but for clarity of exposition, the two person case is prioritized.} In such a scenario, the players will tend to bargain for ``how much'' the machine will prioritize their separate interests, so to begin, we need some way to quantify ``how much'' each player is prioritized.

For instance, one might model the machine as maximizing the expected value, given its observations, of some utility function $U$ of the environment that equals a weighted sum 
\begin{equation}\label{eqn:harsanyi}
w^1U^1 + w^2U^2
\end{equation}
of the players' individual utility functions $U^1$ and $U^2$, as Harsanyi's social aggregation theorem \citep{harsanyi1980cardinal} recommends.  Then the bargaining process could focus on choosing the values of the weights $w^i$.  

However, this turns out to be a bad idea.  As we shall see in \prop{impossibility}, this solution form is not generally compatible with Pareto optimality when agents have different beliefs.  Harsanyi's setting does not account for agents having different priors, nor for decisions being made sequentially, after future observations.  In such a setting, we need a new form of solution, exhibited here along with a recursion that characterizes optimal solutions by a process analogous to, but meaningfully different from, Bayesian updating.

\subsection{A POMDP formulation}

Let us formalize the machine's decision-making situation using the structure of a Partially Observable Markov Decision Process (POMDP), as depicted by the Bayesian network in \fig{pomdp}.  (See \citet{russell2003artificial} for an introduction to POMDPs, and \citet{darwiche2009modeling} for an introduction to Bayesian networks.)

At each point in time $i$, the machine will have a policy $\pi_i$ that for each possible sequence of observations $o_{\le i}$ and past actions $a_{<i}$, returns a distribution $\pi_i(- \mid \hist i)$ on actions $a_i$, which will then be used to generate an action $a_i$ with probability $\pi(a_i \mid \hist i)$.  In \fig{pomdp}, the part of the Bayes net governed by the machine's policy is highlighted in green.

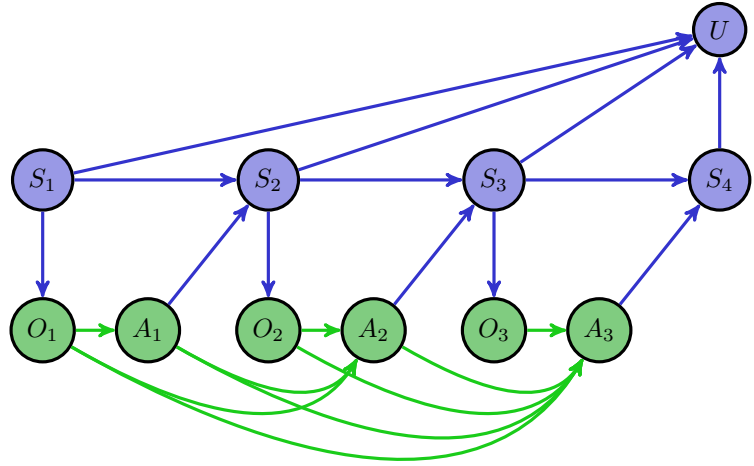
\begin{figure*}
\pgfmathtruncatemacro{\len}{2}
\pgfmathtruncatemacro{\lenn}{\len + 1}
\pgfmathtruncatemacro{\lennn}{\len + 2}
\begin{center}
\begin{tikzpicture}[node distance=1cm, auto, very thick, >=stealth']

  \foreach \x in {1,...,\lennn}{
        \node [draw, circle, fill=envNodeColor]  (s\x) at (3*\x,2) {$S_\x$};
        \node [color=envArrowColor, below left = -0.1cm and 0.3cm of s\x] (T\x) {};
        }

  \foreach \x in {1,...,\lenn}{
       \node [draw, circle, fill=intNodeColor]  (o\x) at (3*\x,0) {$O_\x$};
       \node [draw, circle, fill=intNodeColor]  (a\x) at (3*\x+1.4,0) {$A_\x$};
       }

  \foreach \x in {1,...,\lenn}{
      \pgfmathtruncatemacro{\xn}{\x +1 }
      \draw[->,envArrowColor] (s\x) to node [right] {} (o\x);
      \draw[->,intArrowColor] (o\x) to node {} (a\x);
      \draw[->,envArrowColor] (a\x) to node {} (s\xn);
      \draw[->,envArrowColor] (s\x)--(s\xn);
      }
      
  \foreach \x in {1,...,\len}{
      \pgfmathtruncatemacro{\xn}{\x +1 };
      \foreach \y in {\xn,...,\lenn}{
          \draw[->,intArrowColor] (o\x) to [out=330,in=240] node {} (a\y);
          \draw[->,intArrowColor] (a\x) to [out=330,in=240] node {} (a\y);
          }
      }
  
  \node [draw, circle, fill=envNodeColor] (u) at (3*\lennn,4) {$U$};

  \foreach \x in {1,...,\lennn}{
          \draw[->,envArrowColor] (s\x) to node {} (u);  
  }

\end{tikzpicture}
\end{center}
\caption{A POMDP of length $n=\lenn$}
\label{fig:pomdp}
\end{figure*}

\paragraph{Common knowledge assumptions.}  It is assumed that the players will have common knowledge of the policy $\pi = (\pi_1,\ldots,\pi_n)$ they select for the machine to implement, but that the players may have different beliefs about how the environment works, and of course different utility functions.  It is also assumed that the players have common knowledge of one another's posterior.  

\emph{This assumption is critical.}  During a bargaining process, one should expect players' beliefs to update in response to one another's behavior.   Assuming common knowledge of posteriors means that the players have reached an equilibrium where, each knowing what the other believes, does not wish to further update her own beliefs.\footnote{It is enough to assume the players have reached a ``persistent disagreement'' that cannot be mediated by the machine in some way.  Future work should design solutions for facilitating the process of attaining common knowledge, or to obviate the need to assume it.}

We encode each player $j$'s outlook as a POMDP, $D^j = (\Ss^j,\Aa,T^j, U^j, \Oo, \Omega^j, n)$, which simultaneously represents that player's beliefs about the environment, and the player's utility function.
\begin{itemize}
\item $\Ss^j$ represents a set of possible states $s$ of the environment,
\item $\Aa$ represents the set of possible actions $a$ available to the machine,
\item $T^j$ represents the conditional probabilities player $j$ believes will govern the environment state transitions, i.e., $\PP^j(s_{i+1}\mid s_i a_i)$,
\item $U^j$ represents player $j$'s utility function from sequences of environmental states $(s_1,\ldots,s_n)$ to $\RR$; for the sake of generality, $U^j$ is \emph{not assumed} to be additive over time, as reward functions often are, 
\item $\Oo$ represents the set of possible observations $o$ of the machine, 
\item $\Omega^j$ represents the conditional probabilities player $j$ believes will govern the machine's observations, i.e., $\PP^j(o_i\mid s_i)$, and
\item $n$ is the number of time steps.
\end{itemize}
Thus, player $j$'s subjective probability of an outcome $(\seq s, \seq o, \seq a)$, for any $\seq s \in (\Ss^j)^n$, is given by a probability distribution $\PP^j$ that takes $\pi$ as a parameter:
\begin{equation}\label{eqn:pomdp}
\PP^j(\seq s, \seq o, \seq a ; \pi) := \PP^j(s_1) \cdot \prod_{i=1}^{n} 
\PP^j(o_i \mid s_i)\, \pi(a_i \mid \hist i)\, \PP^j(s_{i+1} \mid s_ia_i)
\end{equation}

As such, the POMDPs $D^1$ and $D^2$ are ``compatible'' in the following sense:

\begin{definition}[Compatible POMDPs]  We say that two POMDPs, $D^1$ and $D^2$, are \emph{compatible} if any policy for one may be viewed as a policy for the other, i.e., they have the same set of actions $\Aa$ and observations $\Oo$, and the same number of time steps $n$.
\end{definition}

\subsection{Pareto optimal policies}

In this context, where a policy $\pi$ may be evaluated relative to more than one POMDP, we use superscripts to represent which POMDP is governing the probabilities and expectations, e.g.,
\[
\EE^j[U^j; \pi] := \sum_{\seq s \in (\Ss^j)^n} \PP^j(\bar s; \pi) U^j(\bar s)
\]
represents the expectation in $D^j$ of the utility function $U^j$, assuming policy $\pi$ is followed.
\begin{definition}[Pareto optimal policies]  A policy $\pi$ is \emph{Pareto optimal} for a compatible pair of POMDPs $(D^1,D^2)$ if for any other policy $\pi'$, either
\[
\EE^1[U^1;\pi] \ge \EE^1[U^1;\pi'] \textrm{\quad or \quad} \EE^2[U^2;\pi] \ge \EE^2[U^2;\pi'].
\]
\end{definition}
It is assumed that, during negotiation, the players will be seeking a Pareto optimal policy for the machine to follow, relative to the POMDPs $D^1$ and $D^2$ describing each player's outlook.

\paragraph{Policy mixing assumption.}  It is also assumed that during the agent's first action (or before it), the agent has the ability to generate and store some random numbers in the interval $[0,1]$, called a random seed, that will not affect the environment except through other features of its actions.  Then, given any two policies $\pi$ and $\pi'$ and a scalar $p\in[0,1]$ we may construct a third policy, 
\[
p\pi + (1-p)\pi',
\]
 that decides with probability $p$ (before receiving any inputs) to use policy $\pi$ for generating all of its future actions, and otherwise uses policy $\pi'$.  (This is a ``once and for all'' decision; the agent does not flip-flop between $\pi$ and $\pi'$ once the decision is made.)  Mixtures of more than two policies are defined similarly.  With this formalism, whenever $\sum_k \alpha_k = 1$ and each $\alpha_k\ge 0$, we have
\begin{equation}\label{eqn:linearity}
\EE^j\left[U^j ; \sum_k \alpha_k \pi_k\right] = \sum_k \alpha_k \EE^j[U^j ; \pi_k].
\end{equation}

\begin{lemma}\label{lem:pareto}
A policy $\pi$ is Pareto optimal to players $1$ and $2$ if and only if there exist weights $w^1,w^2\geq 0$ with $w_1+w_2=1$ such that
\begin{equation}\label{eqn:pareto}
\pi \in \argmax_{\pi^* \in\Pi}\left(w^1\EE^1[U^1;\pi^*] + w^2\EE^2[U^2; \pi^*]\right)
\end{equation}
\end{lemma}

\begin{proof}
The mixing assumption gives the space of policies $\Pi$ the structure of a convex space that the maps $\EE^j[U^j; -] $ respect by \eqn{linearity}.  This ensures that the image of the map $f:\Pi\to\RR^2$ given by
\[
f(\pi) := \left(\EE^1[U^1;\pi],\; \EE^2[U^2;\pi]\right)
\]
is a closed, convex polytope.  As such, a point $(x,y)$ lies on the Pareto boundary of $\image(f)$ if and only if there exist nonnegative weights $(w^1,w^2)$, not both zero, such that 
\[
(x,y) \in \argmax_{(x^*,y^*)\in \image(f)} \left(w^1x^* + w^2y^*\right)
\]
After normalizing $w^1+w^2$ to equal $1$, this implies the result.
\end{proof}

\subsection{A reprioritization mechanism that resembles Bayesian updating}

We shall soon see that any Pareto optimal policy $\pi$ must favor, as time progresses, optimizing the \emph{utility} of whichever player's \emph{beliefs} were a better predictor of the machine's inputs.  This phenomenon turns out to algebraically resemble Bayesian updating, but is quite different in its meaning.  Nonetheless, it is most easily shown to occur by a precise analogy to Bayesian updating in a third POMDP constructed from the outlooks of players 1 and 2, as follows.

For any weights, $w^1,w^2\ge 0$ with $w^1+w^2=1$, we define a new POMDP that works by flipping a $(w^1,w^2)$-weighted coin, and then running $D^1$ or $D^2$ thereafter, according to the coin flip.  Explicitly, 

\begin{definition}[POMDP mixtures]\label{defn:mixture} Let $D^1$ and $D^2$ be compatible POMDPs, with parameters 
$D^j = (\Ss^j,\Aa,T^j, U^j, \Oo, \Omega^j, n)$.  Define a new POMDP compatible both, denoted $D=w^1D^1 + w^2D^2$, with parameters 
$D^j = (\Ss,\Aa,T, U, \Oo, \Omega, n)$, as follows:
\begin{itemize}
\item $\Ss:= \{(j,s) \mid j\in\{1,2\}, s\in\Ss^j\}$,
\item The environmental transition probabilities $T$ are given by
\begin{align*}
\PP\left((j,s_1)\right) :=&\: w^j\cdot \PP^j(s_1)\\
\intertext{for any initial state $s_1\in\Ss^j$, and thereafter,}
\PP\left((j',s_{i+1}) \mid (j,s_i), a_i \right) :=&\;
\begin{cases}
\PP^j\left(s_{i+1} \mid s_ia_i\right) &\mbox{ if $j'=j$}\\
0 & \mbox { if $j'\neq j$}
\end{cases}
\end{align*}
Hence, the value of $j$ will be constant over time, so a full history for the environment may be represented by a pair
\[
(j,\seq s) \in \{1\}\times (\Ss^1)^n \cup \{2\}\times(\Ss^2)^n.
\]
Let $\bool $ denote the boolean random variable that equals whichever constant value of $j$ obtains, so then
\[
\PP(\bool =j) = w^j
\]
\item The utility function $U$ is given by
\[
U(j,\seq s) := U^j(\seq s)
\]
\item The observation probabilities $\Omega$ are given by
\[
\PP\left(o_i \mid (j,s_i)\right) := \PP(\bool =j) \cdot \PP^j(o_i \mid s_i)
\]
In particular, the policy does not observe directly whether $j=1$ or $j=2$.
\end{itemize}

\end{definition}
\begin{figure*}
\pgfmathtruncatemacro{\len}{2}
\pgfmathtruncatemacro{\lenn}{\len + 1}
\pgfmathtruncatemacro{\lennn}{\len + 2}
\begin{center}
\begin{tikzpicture}[node distance=1cm, auto, very thick, >=stealth']
        \node [draw, circle, fill=boolNodeColor]  (b) at (3,4) {$\bool $};
        
  \foreach \x in {1,...,\lennn}{
        \node [draw, circle, fill=envNodeColor]  (s\x) at (3*\x,2) {$S_\x$};
        \draw [->,boolArrowColor] (b) to node {} (s\x);
        \node [color=envArrowColor, below left = -0.1cm and 0.3cm of s\x] (T\x) {};
        }

  \foreach \x in {1,...,\lenn}{
       \node [draw, circle, fill=intNodeColor]  (o\x) at (3*\x,0) {$O_\x$};
       \draw [->,boolArrowColor] (b) to [out=270,in=145] node [right] {} (o\x);
       \node [draw, circle, fill=intNodeColor]  (a\x) at (3*\x+1.4,0) {$A_\x$};
       }

  \foreach \x in {1,...,\lenn}{
      \pgfmathtruncatemacro{\xn}{\x +1 }
      \draw[->,envArrowColor] (s\x) to node [right] {} (o\x);
      \draw[->,intArrowColor] (o\x) to node {} (a\x);
      \draw[->,envArrowColor] (a\x) to node {} (s\xn);
      \draw[->,envArrowColor] (s\x)--(s\xn);
      }
      
  \foreach \x in {1,...,\len}{
      \pgfmathtruncatemacro{\xn}{\x +1 };
      \foreach \y in {\xn,...,\lenn}{
          \draw[->,intArrowColor] (o\x) to [out=330,in=240] node {} (a\y);
          \draw[->,intArrowColor] (a\x) to [out=330,in=240] node {} (a\y);
          }
      }
  
  \node [draw, circle, fill=envNodeColor] (u) at (3*\lennn,4) {$U$};
  \draw[->,boolArrowColor] (b) to node [right] {} (u);

  \foreach \x in {1,...,\lennn}{
          \draw[->,envArrowColor] (s\x) to node {} (u);  
  }
\end{tikzpicture}
\end{center}
\caption{A POMDP (mixture) of length $n=\lenn$ initialized by a Boolean $\bool $}
\label{fig:pomdp2}
\end{figure*}
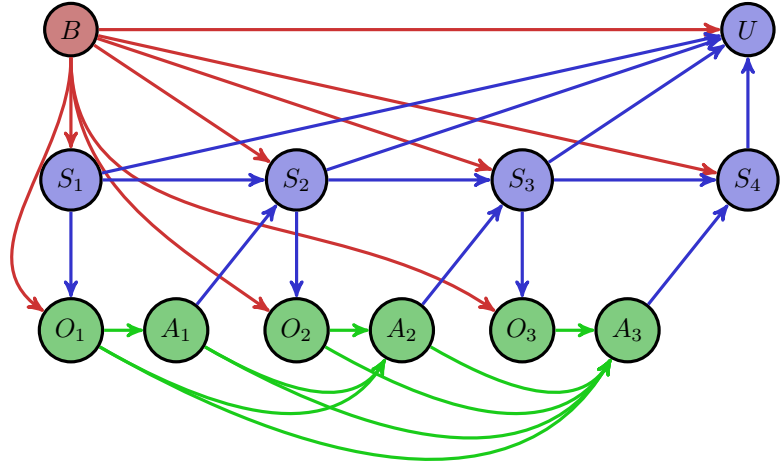

The POMDP mixture $D=w^1D^1 + w^2D^2$ can be depicted with a Bayes net by adding an additional environmental node for $\bool $ in the diagram of $D^1$ and $D^2$ (see \fig{pomdp2}).  Indeed, given any policy $\pi$, the expected payoff of $\pi$ in $w^1D^1+w^2D^2$ is exactly
\begin{align*}
&\; \PP(\bool =1)\cdot\EE(U \mid \bool =1 ; \pi) + \PP(\bool =2)\cdot\EE(U \mid \bool =2 ; \pi)\\
= &\; w^1\EE^2(U^1; \pi) + w^2\EE^2(U^2 ; \pi)
\end{align*}
Therefore, using the above definitions, \lem{pareto} may be restated in the following equivalent form:

\begin{lemma}\label{lem:mixture}
Given a pair $(D^1,D^2)$ of compatible POMDPs, a policy $\pi$ is Pareto optimal for that pair if and only if there exist weights $w^j$ such that $\pi$ is an optimal policy for the single POMDP given by $w^1D^1+w^2D^2$.  
\end{lemma}

\subsection{A recursive Pareto optimality condition}

Expressed in the form of \eqn{pareto}, it might not be clear how a Pareto optimal policy makes use of its observations over time, aside from storing them in memory.  For example, is there any sense in which the machine carries ``beliefs" about the environment that it ``updates" at each time step?  \lem{mixture} allows us to answer this and related questions by translating theorems about single POMDPs into theorems about compatible pairs of POMDPs.  

If $\pi$ is an optimal policy for a single POMDP, at any time step $i$, optimality of the action distribution $\pi_i(-\mid\hist{i})$ can be characterized without reference to the previous policy components 
$(\pi_1,\ldots,\pi_{i-1})$, nor to $\pi_i(-\mid\althist{i})$ for any alternate history $\althist{i}$.\footnote{This fact can used to justify why the ``sunk cost'' fallacy is indeed a fallacy.}   To express this claim in an equation, Pearl's ``$do()$" notation \citep{pearl2009causality} comes in handy:
\begin{definition}[``do'' notation]
\[
\PP^j(\seq o \mid do(\seq a)) := \sum_{\seq s \in (\Ss^j)^n} \PP^j(s_1) \cdot \prod_{i=1}^n
\PP^j(o_i \mid s_i)\, \PP^j(s_{i+1} \mid s_ia_i)
\]
\end{definition}

This expression is the same as the probability of $(\seq o, \seq a)$ when $\pi$ is the constant policy that places probability $1$ on the action sequence $\seq a$.

\begin{proposition}[Classical separability]\label{prop:separability}
If $D$ is a POMDP described by conditional probabilities $\PP(-\mid -)$ and utility function $U$ (as in \eqn{pomdp}), then a policy $\pi$ is optimal for $D$ if and only if for each time step $i$ and each observation/action history $\hist{i}$, the action distribution $\pi_i(-\mid \hist{n})$ satisfies the following backward recursion:
\begin{align*}
& \pi_i(-\mid\hist{i}) \in \argmax_{\alpha\in\Delta A} \left(\vphantom{\left(()\right)}\right.\\
& \left.\quad \PP(\dohist{i})\cdot \EE[U \mid \hist{i} ;\; a_n\sim \alpha; \; \pi_{i+1},\ldots,\pi_{n}]\right)
\end{align*}
This characterization of $\pi_i(\hist{i})$ does not refer to $\pi_1,\ldots,\pi_{i-1}$, nor to 
$\pi_i(\althist{i})$ for any alternate history $\althist{i}$.
\end{proposition}

\begin{proof}
This is a standard property of POMDP solutions.
\end{proof}

It turns out that Pareto optimality can be characterized in a similar way by backward recursion from the final time step.  The resulting recursion reveals a pattern in how the weights on the players' conditionally expected utilities must change over time, which is the main result of this paper:

\begin{theorem}[Pareto optimal policy recursion]\label{thm:main}
Given a pair $(D^1,D^2)$ of compatible POMDPs of length $n$, a policy $\pi$ is Pareto optimal if and only if its components $\pi_i$ for $i\le n$ satisfy the following backward recursion for some pair of weights $w^1,w^2\geq 0$ with $w^1+w^2=1$:
\begin{align*}
&\pi^i(-\mid\hist{i}) \in \argmax_{\alpha\in\Delta A} \left(\vphantom{\left(\left(\right)\right)}\right.\\
& \left.w^1 \PP^1\left(\dohist{i}\right)\cdot \EE^1[U^1 \mid \hist{i} a_i ;\; a_i\sim \alpha;\; \pi_{i+1},\ldots,\pi_n]\right.\\
+ & \left.w^2 \PP^2\left(\dohist{i}\right)\cdot \EE^2[U^2 \mid \hist{i} a_i;\; a_i\sim \alpha;\; \pi_{i+1},\ldots,\pi_n] \right)
\end{align*}
In words, to achieve Pareto optimality, the machine must
\begin{enumerate}
\item use each player's own beliefs when estimating the degree to which a decision favors that player's utility function, and
\item shift the relative priorities of the players' expected utilities in the machine's decision objective over time, by a factor proportional to how well the players predict the machine's inputs.
\end{enumerate}
\end{theorem}

\begin{proof}
By \lem{mixture}, the Pareto optimality of $\pi$ for $(D^1,D^2)$ is equivalent to its classical optimality for
$w^1D^1 + w^2D^2$ for some $(w^1,w^2)$, which by \prop{separability} is equivalent to satisfying the following backward recursion (writing $\PP$ and $\EE$ for probabilities and expectations in $w^1D^1 + w^2D^2$):
\begin{align*}
&\pi^i(-\mid\hist{i}) \in \argmax_{\alpha\in\Delta A} \left(\vphantom{\left(\left(\right)\right)}\right.\\
& \left.\PP(\bool =1)\cdot \PP\left(\dohist{i}\right)\cdot \EE[U \mid \hist{i} a_i ;\; a_i\sim \alpha;\; \pi_{i+1},\ldots,\pi_n]\right.\\
+ & \left.\PP(\bool =2)\cdot \PP\left(\dohist{i}\right)\cdot \EE[U \mid \hist{i} a_i;\; a_i\sim \alpha;\; \pi_{i+1},\ldots,\pi_n] \right).
\end{align*}
By \defn{mixture}, the expression inside the $\argmax$ equals 
\begin{align*}
& w^1 \PP^1\left(\dohist{i}\right)\cdot \EE^1[U^1 \mid \hist{i} a_i ;\; a_i\sim \alpha;\; \pi_{i+1},\ldots,\pi_n]\\
+ & w^2 \PP^2\left(\dohist{i}\right)\cdot \EE^2[U^2 \mid \hist{i} a_i;\; a_i\sim \alpha;\; \pi_{i+1},\ldots,\pi_n]
\end{align*}
hence the result.
\end{proof}

When the players have the same beliefs, they aways assign the same probability to the machine's inputs, so the weights on their respective expectations do not change over time.  In this case, Harsanyi's utility aggregation formula is recovered as a special instance:

\begin{corollary}[Harsanyi's utility aggregation formula]\label{cor:harsanyi}
Suppose that players 1 and 2 share the same beliefs about the environment, i.e., the pair $(D^1,D^2)$ of compatible POMDPs agree on all parameters except the players' utility functions $U^1\neq U^2$. Then a policy $\pi$ is Pareto optimal if and only if there exist weights $w^1,w^2\geq 0$ with $w^1+w^2=1$ such that for $i\le n$, $\pi_i$ satisfies
\[
\pi^i(-\mid\hist{i}) \in \argmax_{\alpha\in\Delta A} \left(\EE[w^1U^1+w^2U^2] \mid \hist{i} a_i ;\; a_i\sim \alpha;\; \pi_{i+1},\ldots,\pi_n]\right)
\]
where $\EE=\EE^1=\EE^2$ denotes the shared expectations of both players.
\end{corollary}
\begin{proof}
Setting $\EE=\EE^1=\EE^2$ in \thm{main}, factoring out the common coefficient $\PP^1\left(\dohist{i}\right)=\PP^2\left(\dohist{i}\right)$, and applying linearity of expectation yields the result.
\end{proof}

\subsection{Comparison to na\"{i}ve utility aggregation}\label{sec:comparison}

To see the necessity of the $\PP^j$ terms that shift the expectation weights in \thm{main} over time, let us compare it with the behavior of an alternative optimization criterion that maximizes a fixed linear combination of expectations.

\paragraph{A cake-splitting scenario.}  The parameters of this scenario are laid out in Table \ref{table:scenario}, and described as follows:

Alice (Player 1) and Bob (Player 2) are about to be presented with a cake which they can choose to split in half to share, or give entirely to one of them.  They have (built or purchased) a robot that will make the cake-splitting decision on their behalf.  Alice's utility function returns $0$ if she gets no cake, $20$ if she gets half a cake, or $30$ if she gets a whole cake.  Bob's utility function works similarly.  

However, Alice and Bob have slightly different beliefs about how the environment works.  They both agree on the state of the environment that the robot will encounter at first: a room with a cake in it ($S_1=\text{``cake''}$).  But Alice and Bob have different predictions about how the robot's sensors will perceive the cake: Alice thinks that when the robot perceives the cake, it is $90\%$ likely to appear with a red tint ($O_1=\text{``red"}$), and $10\%$ likely to appear with a green tint ($O_1=\text{``green"}$), whereas Bob believes the exact opposite.  In either case, upon seeing the cake, the robot will either give Alice the entire cake ($A_1=S_1=\text{(all, none)}$), split the cake half-and-half ($A_1=S_1=\text{(half, half)}$), or give Bob the entire cake ($A_1=S_1=\text{(none, all)}$).  Moreover, Alice and Bob have common knowledge of all these facts.

\begin{table}[htbp]
\centering
\begin{tabular}{*{7}{|c}|}
\hline
$S_1$ & $O_1$& $\PP^1(O_1\mid S_1)$ & $\PP^2(O_1\mid S_1)$ & $A_1=S_1$ & $U^1$ & $U^2$ \\ \hline
\multirow{ 6}{*}{cake}
& \multirow{ 3}{*}{red}
& \multirow{ 3}{*}{$90\%$}
& \multirow{ 3}{*}{$10\%$}
& (all, none) & 30 & 0 \\ 
&&&& (half, half) & 20 & 20 \\ 
&&&& (none, all) & 0 & 30 \\ 
\cline{2-7}
& \multirow{ 3}{*}{green}
& \multirow{ 3}{*}{$10\%$}
& \multirow{ 3}{*}{$90\%$}
& (all, none) & 30 & 0 \\ 
&&&& (half, half) & 20 & 20 \\ 
&&&& (none, all) & 0 & 30 \\ 
\hline
\end{tabular}
\caption{An example scenario wherein a Pareto optimal policy undergoes priority shifting}
\label{table:scenario}
\end{table}

Now, consider the following Pareto optimal policy that favors Alice (Player 1) when $O_1$ is red, and Bob (Player 2) when $O_1$ is green:
\begin{align*}
&\hat\pi(- \mid \text{red}) = 100\%\text{(all, none)}\\
&\hat\pi(- \mid \text{green}) = 100\%\text{(none, all)}
\end{align*}
This policy can be viewed intuitively as a bet between Alice and Bob about the value of $O_1$, and is highly appealing to both players:
\begin{align*}
\EE^1[U^1; \hat\pi] &= 90\%(30) + 10\%(0) = 27\\
\EE^2[U^2; \hat\pi] &= 10\%(0) + 90\%(30) = 27
\end{align*}
In particular, $\hat\pi$ is more appealing to both Alice and Bob than an agreement to deterministically split the cake (half, half).  However, 

\begin{proposition}\label{prop:impossibility}
The Pareto optimal strategy $\hat\pi$ above cannot be implemented by any machine that na\"{i}vely maximizes a fixed-over-time linear combination of the conditionally expected utility of the two players, i.e., by any policy $\pi$ satisfying 
\begin{equation}\label{eqn:naive}
\pi(- \mid o_1) \in \argmax_{\alpha\in\Delta A}\left(r\cdot\EE^1[U^1 \mid o_1; a_1\sim\alpha] + (1-r)\cdot\EE^2[U^2 \mid o_1; a_1\sim\alpha]\right)
\end{equation}
for some fixed $r\in[0,1]$.  Moreover, every such policy $\pi$ is strictly worse than $\hat\pi$ in expectation to one of the players.
\end{proposition} 

This proposition is relatively unsurprising when one considers the policy $\hat\pi$ intuitively as a bet-settling mechanism, and that the nature of betting is to favor different preferences based on future observations.  However, to be sure of this impossibility claim, one must rule out the possibility that the $\hat\pi$ could be implemented by having the machine choose which element of the $\argmax$ in \eqn{naive} to use based on whether the cake appears red or green.  

\begin{proof}[Proof of \prop{impossibility}] Suppose $\pi$ is any policy satisfying \eqn{naive} for some fixed $r$, and consider the following cases for $r$:
\begin{enumerate}
\item If $r < 1/3 $, then $\pi$ must satisfy
\[
\pi(-\mid o_1) = 100\%\text{(none, all)}.
\] 
Here, $\EE^1[U^1 ; \pi] = 0 < 27$, so $\pi$ is strictly worse than $\hat\pi$ in expectation to Alice.

\item If $r = 1/3 $, then $\pi$ must satisfy
\[
\pi(-\mid o_1) = q(o_1)\text{(none, all)} + (1-q(o_1))\text{(half, half)}
\] 
for some $q(o_1)\in[0,1]$ depending on $o_1$.  Here, $\EE^1[U^1 ; \pi] \le 20 < 27$ (with equality when $q(\text{red})=q(\text{green})=1$), so $\pi$ is strictly worse than $\hat\pi$ in expectation to Alice.

\item If $1/3 < r < 2/3 $, then $\pi$ must satisfy
\[
\pi(-\mid o_1)=100\%\text{(half, half)}
\]
Here, $\EE^1[U^1 ; \pi] = \EE^2[U^2 ; \pi] = 20 < 27$, so $\pi$ is strictly worse than $\hat\pi$ in expectation to both Alice and Bob.
\end{enumerate}
The remaining cases, $r=2/3$ and  $r>2/3$, are symmetric to the first two, with Bob in place of Alice and (none, all) in place of (all, none).
\end{proof}

\section{Interpretations}

Theorem \ref{thm:main} shows that a Pareto optimal policy must tend, over time, toward prioritizing the expected {\em utility} of whichever player's \emph{beliefs} best predict the machine's inputs better.  From some perspectives, this is a little counterintuitive: not only must the machine gradually place more predictive weight on whichever player's prior is a better predictor, but it must reward that player by attending more to her utility function as well.  This behavior is not an assumption, but rather is forced to occur by Pareto optimality.  The players \emph{must} agree to this pattern of shifting priority over time, or else they will leave Pareto improvements on the table during the bargaining period when they choose the machine's policy.  

This phenomenon warrants a few interpretations:

\paragraph{Bet settling.}

As discussed in \sect{comparison}, a machine implementing a Pareto optimal policy can be viewed as a kind of bet-settling device.  If Alice is 90\% sure the Four Horsemen will appear tomorrow and Bob is 80\% sure they won't, it makes sense for Alice to ask---while bargaining with Bob for the machine's policy---that the machine prioritize her values more if the Four Horsemen arrive tomorrow, in exchange for prioritizing Bob's values more if they don't.  Both parties will be happy with this agreement in expectation.  As long as it remains possible to redistribute the machine's priorities in a way that resembles an agreeable bet between Alice and Bob, its policy is not yet Pareto optimal.  Thus, \thm{main} can be seen as saying that a Pareto optimality policy goes about settling a bet on the machine's input at each time step, in such a way that no additional bets settlable by the policy are desirable to both players.

\paragraph{Moral realism with Bayesian updating.}

Alternatively, we could take more seriously the interpretation of the weights $w^j$ in \thm{main} as prior ``beliefs'' about the value of the made-up latent variable $\bool $ from \lem{mixture} that simultaneously governs (1) how the environment works, and (2) what utility function is ``correct'' to pursue.  This interpretation is a bit unnatural because, even if the original environmental variables $S_i$ were very grounded in physical reality, the abstract variable $\bool $ in \lem{mixture} is merely a fiction conjured up to imply a correlation between ``is'' and ``ought'': namely, that either the world \emph{is} governed by $\PP^1$, and $U^1$ \emph{ought} to be optimized, or the world \emph{is} governed by $\PP^2$, and $U^2$ \emph{ought} to be optimized.  This occurs even if each of the two players treats their beliefs and utilities completely separately (i.e., even if they apply Bayesian updating only to their beliefs, and keep their utility functions fixed).

Mixing ``is'' and ``ought'' in this way is often considered a type error.  Nonetheless, many humans report an intuitive sense that there are objective, right-and-wrong answers to moral questions that can be answered by observing the world.  If a human is implicitly and approximately acting in a Pareto optimal fashion for a mixture of belief/utility outlooks $D^1, \ldots, D^k$, then the process of ``updating'' to favor a certain utility function might feel, from the inside, like ``finding an answer" to a moral question.

\section{Current limitations and future directions}

The eventual aim of this work is to facilitate the cooperative development and deployment of advanced AI systems, by simplifying the process of bargaining for shared control of such systems, and by making collaborative outcomes generally easier to implement and more attractive.  For this purpose, while Pareto optimality is a desirable condition to aim for, it is not an adequate solution concept on its own.  Indeed, the policy ``maximize player 1's utility function, without regard for player 2'' is Pareto optimal, yet is clearly not the sort of solution one would expect two nations to agree upon.  This and at least several other issues must be addressed to build a satisfactory negotiation framework, exhibited below in order of increasing difficulty as estimated by the author.

\paragraph{1. BATNA dominance.}  In any bargaining situation, each player has a ``best alternative to negotiated agreement'', or BATNA, that she expects to obtain if the other player chooses not to cooperate.  The characterization of Pareto optimality given in \thm{main} {\em does not} account for the players' BATNAs.  Given a facet of the Pareto boundary, specified by the maximization of a linear function with weights $(w^1,w^2)$, a policy $\pi$ satisfying \thm{main} will yield an expectation pair $\epair$ lying on that facet.  Thus, the bargaining problem has been reduced to choosing an appropriate facet of the Pareto boundary.  But suppose not all points on the chosen facet lie above both players' BATNAs.  Then, in order to satisfy the individual rationality of the players, the policy {\em should} target a more specific subset of that facet.

\paragraph{2. Targeting specific expectation pairs.}  If a specific target value for the expectation pair $\epair$ is desired, unless that pair is a vertex of the Pareto region (e.g., perhaps the boundary is curved), the best that following the recursion of \thm{main} ensures is a point on the same facet.  The ability to target a specific pair would solve not only BATNA dominance (1), but also help achieve other fairness or robustness criteria that might arise from bargaining.  One approach would be to make a small modification to the players' utility functions to ensure that the resulting Pareto boundary is curved, thereby avoiding this problem at the cost of a tiny utility adjustment.  Choosing a simple form for the adjustment that is amenable to formal proof would be a natural next step in this direction.

\paragraph{3. Information trade.}  Our algorithm implicitly favors whichever player best predicts the machine's input history, given its action history.  This makes sense when the players have common knowledge of each other's priors and observations, at which point they have already had the opportunity to update on each other's views and chosen not to.  This is unrealistic if Alice knows that Bob has made observations in the past that Alice did not.  In that case, Alice will view Bob's beliefs as containing valuable information that ought to shift her prior.  She may wish to bargain with Bob for access to that information in order to improve her own ability to optimize the machine's policy.  Perhaps she would concede some control over the machine (by reducing her weight, $w^1$) in exchange for information provided by Bob to improve her beliefs.  An efficient procedure to naturally facilitate this sort of exchange would be complimentary \thm{main}.  One approach would be to have each player express their posterior as a function of the other's, and use a fixed point theorem to choose a stable pair of posteriors.  However, many questions arise about this method when there are multiple fixed points.

\paragraph{4. Learning priors and utility functions.}  It is notoriously difficult to explicitly specify one's utility function $U$ to a machine, so in practice, one must choose a method enabling the machine to learn the utility function.  Cooperative inverse reinforcement learning (CIRL) \citep{hadfield2016cooperative} exhibits such a framework, and reduces the problem to solving a POMDP.   In CIRL, a human and a robot play a cooperative game wherein both players aim to maximize the human's utility function $U$, but the robot is uncertain about $U$ and must infer it from the human.  Moreover, the human and robot have common knowledge of this situation, so the human may engage in ``teaching'' behavior to help the robot along.  Such dynamics must be accounted for in a satisfactory treatment of negotiation for a machine's priorities.  In addition, the players' priors should probably also be learned by a machine in some way rather than explicitly specified.

\paragraph{5. Incentive compatibility.}  Assuming any particular method for learning players' priors and utility functions, a question arrises as to whether it incentivizes players to represent their beliefs and utilities honestly.  For example, Alice  may have some incentive to exaggerate her estimation of her BATNA in the positive direction, to motivate Bob to ``sweeten the deal'' by conceding her a higher priority $w^1$ in the recursion of \thm{main}.  As well, players might also have incentives to alter their reported beliefs in order to exaggerate the degree to which the machine's decisions will affect their utilities.  A satisfactory learning method should rule out or otherwise cope with this phenomenon.  A great deal of literature already exists on incentive compatibility, as begun by \citet{hurwicz1972informationally}, \citet{myerson1979incentive}, and \citet{myerson1983efficient}, which should offer a good start.

\paragraph{6. Naturalized decision theory.} The POMDP setting used here is ``Cartesian'' in that it assumes a clear divide between the machine's inner workings and its environment.  This is highly inappropriate when the machine may be copied or simulated; it may wind up in Newcomb-like problems as in \citet{soares2015toward}, and very strange cooperative equilibria may exist between its copies, such as in \citet{critch2016parametric}.  Instead, one should assume a ``naturalized'' model of the problem where the machine is part of its environment, as in \citet{fallenstein2015reflective}.  Some attempts have been made to characterize optimal decision-making in a naturalized setting, e.g., by  \citet{orseau2012space}, but very few theorems to aid in sequential implementation exist (e.g., no analogue of \prop{separability} is known), except possibly for some self-reflective properties exhibited by Garrabrant's logical inductors \citep{garrabrant2016logical} that might be expanded to exhibit relevance to sequential decision-making.   Without a satisfactory model of naturalized decision-making for the machine to follow, the negotiating parties might unwittingly assign the machine a policy vulnerable to Newcomb-like extortions.  On the other hand, a satisfactory resolution would not only help to model the machine's situation, but also that of the players themselves during the negotiation phase.

\section{Conclusion}

Insofar as \thm{main} is not particularly mathematically sophisticated---it employs only basic facts about convexity and linear algebra---this suggests there may be more low-hanging fruit to be found in the domain of ``machine implementable social choice theory".  To recapitulate, \thm{main} represents two deviations from the intuition of na\"{i}ve utility aggregation: to achieve Pareto optimality for players with differing beliefs, a machine must (1) use each player's own beliefs in evaluating how well an action will serve that player's utility function, and (2) shift the relative priority it assigns to each player's expected utilities over time, by a factor proportional to how well that player's beliefs predict the machine's inputs.

As a final remark, consider that social choice theory and bargaining theory were both pioneered during the Cold War, when it was particularly compelling to understand the potential for cooperation between human institutions that might behave competitively.  In the coming decades, machine intelligences will likely bring many new challenges for cooperation, as well as new means to cooperate, and new reasons to do so.  As such, new technical aspects of social choice and bargaining, along the lines of this paper, will likely continue to emerge.  In particular, the problems outlined in Section 5 represent areas particularly promising for facilitating cooperative outcomes in the deployment of advanced AI systems, and the present author is seeking collaborations to address them.








\bibliography{main.bib}

\begin{thebibliography}{30}
\providecommand{\natexlab}[1]{#1}
\providecommand{\url}[1]{\texttt{#1}}
\expandafter\ifx\csname urlstyle\endcsname\relax
  \providecommand{\doi}[1]{doi: #1}\else
  \providecommand{\doi}{doi: \begingroup \urlstyle{rm}\Url}\fi

\bibitem[Abbeel and Ng(2004)]{abbeel2004apprenticeship}
Pieter Abbeel and Andrew~Y Ng.
\newblock Apprenticeship learning via inverse reinforcement learning.
\newblock In \emph{Proceedings of the twenty-first international conference on
  Machine learning}, page~1. ACM, 2004.

\bibitem[Armstrong et~al.(2016)Armstrong, Bostrom, and
  Shulman]{armstrong2016racing}
Stuart Armstrong, Nick Bostrom, and Carl Shulman.
\newblock Racing to the precipice: a model of artificial intelligence
  development.
\newblock \emph{AI \& SOCIETY}, 31\penalty0 (2):\penalty0 201--206, 2016.

\bibitem[Baum(2016)]{baum2016promotion}
Seth~D Baum.
\newblock On the promotion of safe and socially beneficial artificial
  intelligence.
\newblock \emph{AI \& SOCIETY}, pages 1--9, 2016.

\bibitem[Bostrom(2014)]{bostrom2014superintelligence}
Nick Bostrom.
\newblock \emph{Superintelligence: Paths, dangers, strategies}.
\newblock OUP Oxford, 2014.

\bibitem[Critch(2016)]{critch2016parametric}
Andrew Critch.
\newblock Parametric bounded lob's theorem and robust cooperation of bounded
  agents.
\newblock \emph{arXiv preprint arXiv:1602.04184}, 2016.

\bibitem[Darwiche(2009)]{darwiche2009modeling}
Adnan Darwiche.
\newblock \emph{Modeling and reasoning with Bayesian networks (Chapter 4)}.
\newblock Cambridge University Press, 2009.

\bibitem[Fallenstein et~al.(2015)Fallenstein, Soares, and
  Taylor]{fallenstein2015reflective}
Benja Fallenstein, Nate Soares, and Jessica Taylor.
\newblock Reflective variants of solomonoff induction and aixi.
\newblock In \emph{International Conference on Artificial General
  Intelligence}, pages 60--69. Springer, 2015.

\bibitem[G{\'a}bor et~al.(1998)G{\'a}bor, Kalm{\'a}r, and
  Szepesv{\'a}ri]{gabor1998multi}
Zolt{\'a}n G{\'a}bor, Zsolt Kalm{\'a}r, and Csaba Szepesv{\'a}ri.
\newblock Multi-criteria reinforcement learning.
\newblock In \emph{ICML}, volume~98, pages 197--205, 1998.

\bibitem[Garrabrant et~al.(2016)Garrabrant, Benson-Tilsen, Critch, Soares, and
  Taylor]{garrabrant2016logical}
Scott Garrabrant, Tsvi Benson-Tilsen, Andrew Critch, Nate Soares, and Jessica
  Taylor.
\newblock Logical induction.
\newblock \emph{arXiv preprint arXiv:1609.03543}, 2016.

\bibitem[Hadfield-Menell et~al.(2016)Hadfield-Menell, Dragan, Abbeel, and
  Russell]{hadfield2016cooperative}
Dylan Hadfield-Menell, Anca Dragan, Pieter Abbeel, and Stuart Russell.
\newblock Cooperative inverse reinforcement learning, 2016.

\bibitem[Harsanyi(1980)]{harsanyi1980cardinal}
John~C Harsanyi.
\newblock Cardinal welfare, individualistic ethics, and interpersonal
  comparisons of utility.
\newblock In \emph{Essays on Ethics, Social Behavior, and Scientific
  Explanation}, pages 6--23. Springer, 1980.

\bibitem[Hurwicz(1972)]{hurwicz1972informationally}
Leonid Hurwicz.
\newblock On informationally decentralized systems.
\newblock \emph{Decision and organization}, 1972.

\bibitem[Hutter(2003)]{hutter2003gentle}
Marcus Hutter.
\newblock A gentle introduction to the universal algorithmic agent
  $\{$AIXI$\}$, 2003.

\bibitem[Myerson(1979)]{myerson1979incentive}
Roger~B Myerson.
\newblock Incentive compatibility and the bargaining problem.
\newblock \emph{Econometrica: journal of the Econometric Society}, pages
  61--73, 1979.

\bibitem[Myerson(2013)]{myerson2013game}
Roger~B Myerson.
\newblock \emph{Game theory}.
\newblock Harvard university press, 2013.

\bibitem[Myerson and Satterthwaite(1983)]{myerson1983efficient}
Roger~B Myerson and Mark~A Satterthwaite.
\newblock Efficient mechanisms for bilateral trading.
\newblock \emph{Journal of economic theory}, 29\penalty0 (2):\penalty0
  265--281, 1983.

\bibitem[Nash(1950)]{nash1950bargaining}
John~F Nash.
\newblock The bargaining problem.
\newblock \emph{Econometrica: Journal of the Econometric Society}, pages
  155--162, 1950.

\bibitem[Ng et~al.(2000)Ng, Russell, et~al.]{ng2000algorithms}
Andrew~Y Ng, Stuart~J Russell, et~al.
\newblock Algorithms for inverse reinforcement learning.
\newblock In \emph{Icml}, pages 663--670, 2000.

\bibitem[Orseau and Ring(2012)]{orseau2012space}
Laurent Orseau and Mark Ring.
\newblock Space-time embedded intelligence.
\newblock In \emph{International Conference on Artificial General
  Intelligence}, pages 209--218. Springer, 2012.

\bibitem[Pearl(2009)]{pearl2009causality}
Judea Pearl.
\newblock \emph{Causality}.
\newblock Cambridge university press, 2009.

\bibitem[Roijers et~al.(2015)Roijers, Whiteson, and Oliehoek]{roijers2015point}
Diederik~M Roijers, Shimon Whiteson, and Frans~A Oliehoek.
\newblock Point-based planning for multi-objective pomdps.
\newblock In \emph{IJCAI 2015: Proceedings of the Twenty-Fourth International
  Joint Conference on Artificial Intelligence}, pages 1666--1672, 2015.

\bibitem[Russell(1998)]{russell1998learning}
Stuart Russell.
\newblock Learning agents for uncertain environments.
\newblock In \emph{Proceedings of the eleventh annual conference on
  Computational learning theory}, pages 101--103. ACM, 1998.

\bibitem[Russell et~al.(2003)Russell, Norvig, Canny, Malik, and
  Edwards]{russell2003artificial}
Stuart Russell, Peter Norvig, John~F Canny, Jitendra~M Malik, and Douglas~D
  Edwards.
\newblock \emph{Artificial intelligence: a modern approach (Chapter 17.1)},
  volume~2.
\newblock Prentice hall Upper Saddle River, 2003.

\bibitem[Shoham and Leyton-Brown(2008)]{shoham2008multiagent}
Yoav Shoham and Kevin Leyton-Brown.
\newblock \emph{Multiagent systems: Algorithmic, game-theoretic, and logical
  foundations}.
\newblock Cambridge University Press, 2008.

\bibitem[Soares and Fallenstein(2015)]{soares2015toward}
Nate Soares and Benja Fallenstein.
\newblock Toward idealized decision theory.
\newblock \emph{arXiv preprint arXiv:1507.01986}, 2015.

\bibitem[Soh and Demiris(2011)]{soh2011evolving}
Harold Soh and Yiannis Demiris.
\newblock Evolving policies for multi-reward partially observable markov
  decision processes (mr-pomdps).
\newblock In \emph{Proceedings of the 13th annual conference on Genetic and
  evolutionary computation}, pages 713--720. ACM, 2011.

\bibitem[Tzeng and Huang(2011)]{tzeng2011multiple}
Gwo-Hshiung Tzeng and Jih-Jeng Huang.
\newblock \emph{Multiple attribute decision making: methods and applications}.
\newblock CRC press, 2011.

\bibitem[Wang(2014)]{wang2014multi}
Weijia Wang.
\newblock \emph{Multi-objective sequential decision making}.
\newblock PhD thesis, Universit{\'e} Paris Sud-Paris XI, 2014.

\bibitem[Wray and Zilberstein(2015)]{wray2015multi}
Kyle~Hollins Wray and Shlomo Zilberstein.
\newblock Multi-objective pomdps with lexicographic reward preferences.
\newblock In \emph{Proceedings of the 24th International Joint Conference of
  Artificial Intelligence (IJCAI)}, pages 1719--1725, 2015.

\bibitem[Zhang and Shah(2014)]{zhang2014fairness}
Chongjie Zhang and Julie~A Shah.
\newblock Fairness in multi-agent sequential decision-making.
\newblock In \emph{Advances in Neural Information Processing Systems}, pages
  2636--2644, 2014.

\end{thebibliography}

\end{document}